\documentclass{article} % For LaTeX2e
\usepackage{iclr2027_conference,times}

\usepackage{amsmath,amsfonts,bm}

\def\eqref#1{equation~\ref{#1}}
\def\1{\bm{1}}

\def\rvu{{\mathbf{i}}}

\def\rvu{{\mathbf{u}}}

\def\rvz{{\mathbf{z}}}

\def\vzero{{\bm{0}}}
\def\vone{{\bm{1}}}
\def\vmu{{\bm{\mu}}}

\def\va{{\bm{a}}}

\def\vp{{\bm{p}}}

\def\vu{{\bm{u}}}
\def\vv{{\bm{v}}}

\def\vx{{\bm{x}}}

\def\vz{{\bm{z}}}

\def\mA{{\bm{A}}}

\def\mI{{\bm{I}}}

\def\mP{{\bm{P}}}

\def\mS{{\bm{S}}}

\def\mX{{\bm{X}}}

\def\mZ{{\bm{Z}}}

\def\mSigma{{\bm{\Sigma}}}

\DeclareMathAlphabet{\mathsfit}{\encodingdefault}{\sfdefault}{m}{sl}
\SetMathAlphabet{\mathsfit}{bold}{\encodingdefault}{\sfdefault}{bx}{n}

\def\sA{{\mathbb{A}}}
\def\sB{{\mathbb{B}}}
\def\sC{{\mathbb{C}}}

\def\sI{{\mathbb{I}}}
\def\sJ{{\mathbb{J}}}

\newcommand{\E}{\mathbb{E}}

\newcommand{\R}{\mathbb{R}}

\usepackage{hyperref}
\usepackage{url}
\usepackage{amsmath,amssymb}
\usepackage{graphicx}
\usepackage{wrapfig}
\usepackage{booktabs}
\usepackage{multirow}
\usepackage[capitalize]{cleveref}   % provides \Cref used in the section files
\crefname{appendix}{Appendix}{Appendices}
\Crefname{appendix}{Appendix}{Appendices}
\usepackage{amsmath,amssymb,amsthm}

\newtheorem{theorem}{Theorem}
\newtheorem{proposition}[theorem]{Proposition}
\newtheorem{lemma}[theorem]{Lemma}
\newtheorem{corollary}[theorem]{Corollary}

\title{ATLAS: Aligned Transport of Latent Structure for Reliable World Model Planning}

\author{Ke Fang\thanks{Equal contribution.}, \ Yupu Yao\footnotemark[1], \ Lu Cheng\\
Pennsylvania State University\\
\texttt{\{kbf5519, yjy5437, lqc5822\}@psu.edu}}

\iclrfinalcopy % Uncomment for camera-ready version, but NOT for submission.
\begin{document}

\renewcommand{\thefootnote}{\fnsymbol{footnote}}
\maketitle
\lhead{ATLAS: Aligned Transport of Latent Structure for Reliable World Model Planning}
% \lhead{Under review as a conference paper at ICLR 2027}
\renewcommand{\thefootnote}{\arabic{footnote}}
\setcounter{footnote}{0}

\begin{abstract}
% Latent world models rely on representation geometry for planning, yet regularizing the latent marginal alone does not determine the state-to-state relationships used for action selection. We show that this can cause planning-relevant novelty structure to be weakened as representations are transformed into the final latent used by the planner. We introduce Aligned Transport of Latent Structure (ATLAS), a training objective that explicitly preserves relational geometry while calibrating the global latent distribution. ATLAS transfers normalized pairwise structure from an informative encoder representation to the planning latent and uses Wasserstein embedding matching (WEMReg) to calibrate its marginal through one-dimensional Wasserstein-2 transport. Our analysis shows that relational preservation and marginal calibration impose non-redundant constraints, and connects finite-candidate planning stability to relational distortion, latent-scale mismatch, and prediction error. Instantiated in LeWM, ATLAS improves mean goal-reaching success across PushT, TwoRoom, and OGBench-Cube on both lower- and higher-novelty evaluation subsets, with the largest gain on higher-novelty TwoRoom episodes. Representation and rollout diagnostics further show stronger novelty-related structure in the planning latent, improved marginal calibration, and lower multi-step prediction error. Together, these results highlight preservation of planning-relevant latent geometry as an important ingredient for reliable world-model planning. Code is available at \url{https://anonymous.4open.science/r/atlas-world-model-72C4/}.
Latent world models rely on representation geometry for planning, yet regularizing the latent marginal alone does not determine the state-to-state relationships used for action selection. We show that this can cause planning-relevant novelty structure to be weakened as representations are transformed into the final latent used by the planner. We introduce Aligned Transport of Latent Structure (ATLAS), a training objective that explicitly preserves relational geometry while calibrating the global latent distribution. ATLAS transfers normalized pairwise structure from an informative encoder representation to the planning latent and uses Wasserstein embedding matching (WEMReg) to calibrate its marginal through one-dimensional Wasserstein-2 transport. Our analysis shows that relational preservation and marginal calibration impose non-redundant constraints, and connects finite-candidate planning stability to relational distortion, latent-scale mismatch, and prediction error. Instantiated in LeWM, ATLAS improves mean goal-reaching success across PushT, TwoRoom, and OGBench-Cube on both lower- and higher-novelty evaluation subsets, with the largest gain on higher-novelty TwoRoom episodes. Representation and rollout diagnostics further show stronger novelty-related structure in the planning latent, improved marginal calibration, and lower multi-step prediction error. Together, these results highlight preservation of planning-relevant latent geometry as an important ingredient for reliable world-model planning. Code is available at \url{https://github.com/Annie969/atlas-world-model}.
\end{abstract}

\section{Introduction}
\label{sec:intro}
\label{sec:intro}

\begin{wrapfigure}{r}{0.46\textwidth}
    \centering
    \vspace{-0.4inch}
    \includegraphics[width=0.44\textwidth]{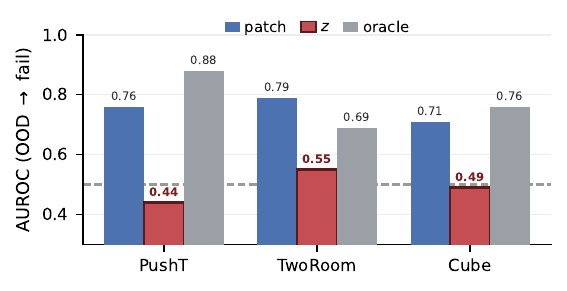}
    \caption{\textbf{An OOD-related diagnostic gap in latent planning on PushT.}
   }
    \label{fig:ood-blind}
    \vspace{-0.8\baselineskip}
\end{wrapfigure}

%   AUROC for predicting planning failure from k-nearest-neighbor noveltyscores in LeWM at goal offset 50. The same diagnostic is substantially more predictive when computed from mean-pooled patch features or true environment states than from the final planning latent. Chance AUROC is 0.5
Latent world models (WMs) plan by predicting where different candidate actions will lead and selecting the action whose predicted outcome is closest to a desired goal \citep{pmlr-v97-hafner19a, ICLR2024_cf73d57b}. The geometry of the learned representation space is therefore part of the planning problem: states that are behaviorally similar should remain appropriately close, while states that differ in important ways should remain distinguishable. This becomes especially important when planning encounters out-of-distribution (OOD) states that are rare or poorly represented in the offline training data.

Many recent WMs learn these representations by predicting future latent features rather than reconstructing raw observations, an approach closely related to joint-embedding predictive architectures (JEPAs) \citep{lecun2022path, 10205476, bardes2024revisitingfeaturepredictionlearning}. A central challenge in such models is preventing representation collapse, where different inputs are mapped to nearly identical features \citep{NEURIPS2020_f3ada80d, 9578004, pmlr-v139-zbontar21a, bardes2024revisitingfeaturepredictionlearning}. Existing approaches address this through anti-collapse regularization that constrains statistical properties of the latent representation, such as its variance, covariance, or overall distribution \citep{bardes2022vicreg,balestriero2025lejepa}. For example, SIGReg encourages latent projections to follow a standard Gaussian distribution \citep{balestriero2025lejepa}. These constraints can keep the representation non-degenerate, but they do not specify which individual states should remain close or far apart. Two latent spaces can therefore have similarly well-behaved global distributions while preserving substantially different relationships between states. For planning, this distinction matters because these relationships directly determine how predicted outcomes are compared with the goal.

We observe this gap in LeWM~\citep{lewm}, a recent JEPA-style WM. We measure how well distance from the training data predicts downstream planning failure at different stages of the learned representation. This signal is substantially stronger in the encoder's mean-pooled patch features than in the final latent used for planning (\Cref{fig:ood-blind}). On PushT \citep{florence2022ibc}, for example, failure-prediction AUROC decreases from $0.76$ in the encoder representation to $0.44$ in the planning latent. Thus, information indicating that a state is unfamiliar is present within the encoder but substantially weakened in the representation ultimately used for planning. This suggests that maintaining a well-behaved latent distribution alone does not ensure preservation of planning-relevant state relationships.

We therefore introduce Aligned Transport of Latent Structure (ATLAS), which explicitly targets two complementary properties of the planning representation. The first is relational preservation: state-to-state relationships present in an informative encoder representation should be retained in the planning latent. Our OOD-recovery objective therefore transfers normalized pairwise distances from the encoder's mean-pooled patch features into the planning representation. The second is marginal calibration: the overall latent distribution should remain well-behaved and resistant to collapse. For this purpose, we introduce Wasserstein embedding matching (WEMReg), which matches one-dimensional projections of the latent distribution to a standard Gaussian using Wasserstein-2 distance~\citep{bonneel2015sliced}. Together with the predictive objective, ATLAS targets three complementary properties of a planning representation: preserving state relationships, maintaining a stable latent distribution, and accurately predicting future states.

Our theoretical analysis shows that relational preservation and marginal calibration impose non-redundant constraints on the latent representation. We further relate representation distortion and rollout prediction error to the stability of finite-candidate goal-conditioned planning.
Across three benchmark datasets, ATLAS improves mean planning success over LeWM on both in-distribution (ID) and OOD episodes, with the largest improvement occurring on OOD episodes. Controlled ablations and representation- and rollout-level diagnostics further support the complementary roles of relational preservation and WEMReg (\Cref{sec:experiments,sec:analysis}).
Overall, our contributions are:

\begin{itemize}
\item We show that a WM can retain useful OOD-related information in its encoder while losing much of that signal in the latent representation used for planning, exposing a gap between anti-collapse regularization and preservation of planning-relevant geometry.

\item We introduce ATLAS, a latent WM that explicitly preserves planning-relevant relational structure while calibrating the global latent geometry through Wasserstein transport.

\item We connect representation distortion and prediction error to planning stability, and evaluate ATLAS across manipulation and navigation tasks using controlled ablations and representation- and rollout-level diagnostics.

\end{itemize}

\section{Related Work}
\label{sec:related}
\textbf{Latent WMs for planning.}
A growing line of work learns WMs in a joint-embedding predictive latent space and
plans by rolling that latent forward. LeWM trains a ViT encoder and an autoregressive latent
predictor end-to-end from pixels and plans by latent Cross-Entropy search~\citep{lewm}. DINO-WM instead fixes a frozen DINOv2 encoder pretrained on
external images and learns dynamics on top of it~\citep{zhou2025dinowm}, PLDM learns latent
dynamics from reward-free offline data~\citep{pldm}, and FAST-WM is a faster LeWM
variant~\citep{fastlewm}. Where these methods change the encoder or the data, ATLAS changes only
the training objective of a from-scratch model, which makes it a controlled modification rather than a different backbone.

% \paragraph{Preventing collapse in joint-embedding learning.}
A fundamental challenge in latent WM is that learned representations can collapse. We need an anti-collapse mechanism, so the encoder does not map every
input to the same point. VICReg does this with explicit variance and covariance penalties on
the embedding~\citep{bardes2022vicreg}, while LeJEPA regularizes the embedding toward an
isotropic Gaussian through a characteristic-function normality test evaluated at a finite set
of frequencies~\citep{balestriero2025lejepa}, the SIGReg mechanism LeWM
adopts~\citep{lewm}. Our diagnosis is that this per-sample Gaussianization,
while effective at preventing collapse, also removes the OOD structure a planner needs. ATLAS
keeps the anti-collapse goal but reaches it with a Wasserstein embedding matching regularization term,
which we show is both lower in design constants and strictly more discriminative than
finite-frequency matching (\Cref{sec:theory}).

\textbf{Optimal transport and sliced distances.}
Matching a distribution to a target through one-dimensional projections underlies the sliced
Wasserstein distance~\citep{bonneel2015sliced}, which replaces an intractable
high-dimensional transport problem with an average of exact one-dimensional ones. We use this
device in the specific form of matching each projected marginal to a standard Gaussian by its
exact quantile transport, which yields a closed-form batch objective and, unlike a
finite-frequency test, penalizes every projected quantile residual (\Cref{sec:theory_sot}).

\textbf{OOD detection and relational distillation.}
Distance to a bank of training features is a standard nonparametric OOD signal: a
$k$-nearest-neighbor distance in representation space separates in- from out-of-distribution
inputs without a parametric density model~\citep{sun2022knnood}. We use this score both as
the diagnostic that reveals the washed-out planning latent (\Cref{sec:intro}) and as the
quantity our theory relates to planning stability (\Cref{sec:theory}). To restore it we
transfer the relational geometry of a stronger representation into a weaker one, in the
spirit of relational knowledge distillation, which matches pairwise structure rather than
absolute activations~\citep{park2019relational}. Here the roles are internal to a single
model: the encoder's own patch tokens act as a fixed anchor whose pairwise-distance geometry
the planning latent is taught to preserve.

\section{ATLAS}
\label{sec:method}
\label{sec:theory}
\label{sec:method-full}

\begin{figure*}[t]
    \centering
    \vspace{-0.3inch}
    \includegraphics[width=0.9\textwidth]{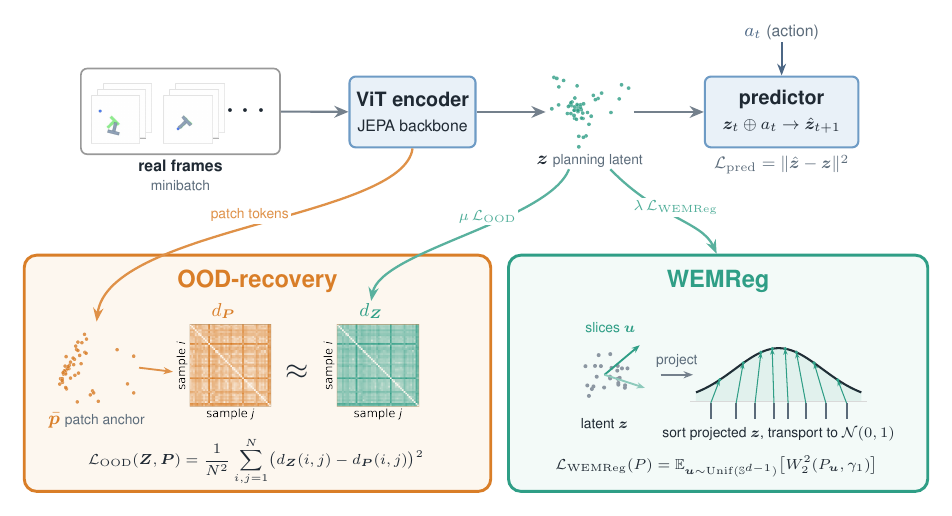}
    \caption{\textbf{Overview of ATLAS.}}
    \label{fig:method}
\end{figure*}

Planning with a latent WM requires a representation that supports accurate prediction under candidate actions, remains well-scaled and non-degenerate, and preserves state-to-state relationships relevant to goal comparison. Standard marginal regularization addresses only the distributional requirement and does not determine which states should remain close or far apart.

ATLAS is designed to address these three requirements (\Cref{fig:method}). The predictive objective learns latent dynamics, WEMReg calibrates the global distribution of the planning latent, and OOD-recovery preserves relational structure from an informative encoder representation. Following the latent predictive architecture of LeWM~\citep{lewm}, a trainable ViT encoder maps each observation $o_t$ to a $d$-dimensional planning latent, $z_t = E_\theta(o_t)$, and a predictor models its evolution under actions:
\begin{equation}
   \widehat{z}_{t+1}
=
P_\phi(z_{t-H+1:t},a_{t-H+1:t}), 
\end{equation}
where the input contains the previous $H$ latent--action pairs and $a_t$ denotes the action at time $t$.

\subsection{Calibrating the latent distribution}
\label{sec:method-sot}
\label{sec:theory_sot}

A useful planning representation should remain well-scaled and non-degenerate. Existing anti-collapse methods impose this through distributional constraints on the latent space. We propose Wasserstein embedding matching (WEMReg), which directly compares the projected distributions \citep{tolstikhin2019wassersteinautoencoders, kolouri2018slicedwassersteinautoencoderembarrassinglysimple}.
For a latent distribution $P$ with finite second moments, we define
\begin{equation}
\mathcal{L}_{\mathrm{WEMReg}}(P)=
\E_{\vu\sim\operatorname{Unif}(\mathbb{S}^{d-1})}
W_2^2(P_\vu,\mathcal{N}(0,1)),
\label{eq:theory_population_sot}
\end{equation}
where $\vu$ is a randomly sampled unit direction and $P_\vu$ denotes the distribution of the scalar projection $\vu^\top \vz$ for $\vz\sim P$. Each direction provides a one-dimensional view of the $d$-dimensional latent distribution, and $W_2^2$ measures how far that projected distribution is from a standard Gaussian. Averaging over random directions therefore encourages Gaussian structure across arbitrary directions, rather than only along individual coordinates.
In practice, we approximate the expectation using $L$ random directions per minibatch, where $L$ is the number of one-dimensional projections being averaged. For each direction, we sort the projected latent values and match them to the corresponding standard-Gaussian quantiles. The resulting one-dimensional Wasserstein cost is therefore determined by the discrepancy between the ordered projected samples and their Gaussian targets.

LeWM instead uses SIGReg for marginal regularization, which tests Gaussianity by matching characteristic functions at a finite set of frequencies. WEMReg targets the same marginal calibration objective, but compares the full projected distributions through Wasserstein distance. This distinction matters because agreement at finitely many characteristic-function frequencies does not uniquely determine a distribution: non-Gaussian discrepancies may remain undetected even when the mean and covariance are correct. We formalize this separation in \Cref{thm:cf_ambiguity}; the precise finite-frequency objective and proof are provided in the Appendix~\ref{app:theory}.

\begin{theorem}[Finite-frequency ambiguity]
\label{thm:cf_ambiguity}
For any finite set of characteristic-function frequencies with positive weights, there exists a non-Gaussian distribution $P$ with zero mean and identity covariance that is indistinguishable from a standard Gaussian under the corresponding finite-frequency objective, while
$\mathcal{L}_{\mathrm{WEMReg}}(P)>0.$
\end{theorem}

Thus, finite-frequency matching can miss non-Gaussian structure that remains visible to projected Wasserstein matching.

\subsection{Preserving relational geometry}
\label{sec:method-ood}
\label{sec:theory_complementarity}

Marginal calibration controls the overall latent distribution, but it does not determine which states should remain close or far apart. We therefore introduce a relational objective that transfers state-to-state geometry from an informative encoder representation to the planning latent.

For frame $i$, let $\bar{\vp}_i=K^{-1}\sum_{k=1}^{K}\vp_{i,k}$ denote the mean of its $K$ final-layer patch tokens. As shown in \Cref{sec:intro}, novelty measured in this representation is substantially more predictive of downstream planning failure than novelty measured in the LeWM planning latent. We therefore use the mean-pooled patch representation as a relational anchor.
For a minibatch of $N$ samples, let $\mZ=[\vz_i^\top]_{i=1}^{N}$ denote the planning latents and $\mP=[\bar{\vp}_i^\top]_{i=1}^{N}$ the corresponding patch representations. For either representation $\mX$, we first standardize each coordinate by its batch standard deviation and compute pairwise Euclidean distances. We then divide all pairwise distances by their batch mean, removing differences in overall scale. Denoting the resulting normalized distance between samples $i$ and $j$ by $d_{\mX}(i,j)$, we define
\begin{equation}
\mathcal{L}_{\mathrm{OOD}}(\mZ,\mP)
=
\frac{1}{N^2}
\sum_{i,j=1}^{N}
\left(
d_{\mZ}(i,j)-d_{\mP}(i,j)
\right)^2.
\label{eq:nov}
\end{equation}
When computing this loss, $\mP$ is treated as stop-gradient. Thus, the patch representation provides a fixed target for each update, while still evolving over training through the shared encoder.

Intuitively, $\mathcal{L}_{\mathrm{OOD}}$ does not force the planning latent to reproduce the patch features themselves. Instead, it preserves their relative geometry: states that are far apart in the anchor representation are encouraged to remain far apart in the planning latent, and similarly for nearby states \citep{Tung_2019_ICCV, Peng_2019_ICCV, park2019relational}.
WEMReg and relational preservation therefore constrain different aspects of the representation.

\begin{proposition}[Non-redundant constraints]
\label{prop:nonredundancy}
For a fixed anchor, two representations can have the same WEMReg loss but different relational losses, or the same relational loss but different WEMReg values.
\end{proposition}

Thus, neither objective determines the other: $\mathcal{L}_{\mathrm{WEMReg}}$ constrains the global latent distribution, whereas $\mathcal{L}_{\mathrm{OOD}}$ constrains normalized state-to-state geometry relative to the encoder anchor.

The complete ATLAS objective is
$
\mathcal{L}
=
\mathcal{L}_{\mathrm{pred}}
+
\lambda\mathcal{L}_{\mathrm{WEMReg}}
+
\mu\mathcal{L}_{\mathrm{OOD}},
\label{eq:total}$
where $\mathcal{L}_{\mathrm{pred}}$ is the latent prediction loss and $\lambda,\mu>0$ control the two representation objectives. The three terms respectively target accurate dynamics prediction, marginal calibration, and preservation of relational geometry.

\subsection{Finite-candidate planning stability}
\label{sec:theory_planning}

At test time, ATLAS uses the same goal-conditioned latent planner as LeWM. Given a goal image $o_g$, CEM searches over candidate action sequences and rolls each candidate forward for $T$ prediction steps\citep{de2005tutorial}. Candidates are scored by the distance between their predicted terminal latent $\widehat{\vz}_T$ and the encoded goal $\vz_g$,
\begin{equation}
|\widehat{\vz}_T-\vz_g|_2^2,
\qquad
\vz_g=E_\theta(o_g).
\end{equation}
CEM repeatedly refines its sampling distribution toward low-cost candidates, executes a prefix of the selected sequence, and replans from the next observation.

We now ask when this finite-candidate selection remains stable. Three factors can perturb the planner: inaccurate terminal predictions, distortion of state-to-state geometry, and mismatch in latent scale. Let $\epsilon_{\mathrm{rel}}$ denote the maximum discrepancy between normalized pairwise distances in the planning latent and the patch-anchor representation, and let $e$ denote the maximum terminal prediction error over the candidate set. We additionally use $\eta$ to quantify latent-scale mismatch. Precise definitions of these quantities and the fixed normalization constants are given in Appendix~\ref{app:theory_details}.

\begin{theorem}[Planning stability]
\label{thm:planning_stability}
For a common reference bank, the mean $k$-nearest-neighbor novelty scores $s$ computed in the planning and anchor representations satisfy
\begin{equation}
|s_{\mZ}(q)-s_{\mP}(q)|
\leq
\epsilon_{\mathrm{rel}}.
\label{eq:theory_knn_stability}
\end{equation}
Moreover, for a fixed finite candidate set and bounded latent-scale mismatch, any candidate selected using predicted latent distance has anchor-geometry regret bounded by
\begin{equation}
2\left[
\frac{e}{\ell}
+
(1+\eta)\epsilon_{\mathrm{rel}}
+
\eta M
\right],
\label{eq:theory_planning_regret}
\end{equation}
where $\ell$ is the latent distance scale and $M$ bounds the candidate-to-goal distance in the anchor geometry.
\end{theorem}

The theorem connects the ATLAS objectives directly to planning stability. The prediction loss reduces terminal rollout error $e$, the relational objective reduces geometry distortion $\epsilon_{\mathrm{rel}}$, and WEMReg controls latent-scale mismatch $\eta$. Together, these terms bound how far the action selected using predicted latent distances can deviate from the action preferred under the anchor geometry. The $k$NN result additionally shows that reducing relational distortion preserves the novelty structure that motivates our relational objective.
Full assumptions, definitions, and proofs are provided in Appendix~\ref{app:theory_details}.

\section{Experiments}
\label{sec:experiments}

\textbf{Datasets.}
\label{sec:exp-data}
We evaluate on three offline control tasks spanning navigation and
manipulation, with different observation structures and dynamics.
\textbf{TwoRoom}~\citep{swm} is a 2D navigation task in which an
agent must move between two rooms through a narrow opening to reach a
target position. \textbf{PushT}~\citep{florence2022ibc,chi2023diffusion}
is a contact-rich 2D manipulation task that requires pushing and
rotating a T-shaped block into a target pose.
\textbf{OGBench-Cube}~\citep{ogbench} extends the evaluation to 3D
robotic manipulation, where a robot arm moves a cube toward a target
configuration. Together, the tasks probe latent planning across
navigation, contact-rich planar manipulation, and visually richer
robotic manipulation. Each model is trained on the corresponding
offline demonstration data; dataset details are given in
\Cref{app:setup}.

\textbf{Baselines.}
\label{sec:exp-baselines}
We compare ATLAS with four latent world models representing different
representation-learning and regularization choices.
\textbf{LeWM}~\citep{lewm} is the closest controlled comparison:
ATLAS uses the same general JEPA-style encoder--predictor setup while
replacing the marginal regularizer and adding relational
preservation. \textbf{DINO-WM}~\citep{zhou2025dinowm} represents a
different representation-learning regime, learning latent dynamics on
top of a frozen DINOv2-S/14 encoder pretrained on large-scale external
image data. Following the comparison protocol of \citet{lewm}, we
remove proprioceptive inputs and retrain its trainable components on
each task for $10$ epochs. \textbf{FAST-WM}~\citep{fastlewm} is an
efficiency-oriented LeWM variant and is evaluated from its released
checkpoints with the original encoder and predictor.
\textbf{PLDM}~\citep{pldm} provides an alternative JEPA-style
training objective based on VICReg-derived regularization; we retrain
it on the same task data for the same $10$-epoch budget.
The comparison to LeWM therefore isolates the effect of the ATLAS
training objective most directly, whereas the remaining baselines
provide broader reference points with different encoders, objectives,
or checkpoint provenance.

\textbf{Settings.}
\label{sec:exp-setting}
We evaluate all methods under a common goal-conditioned planning
protocol, using the same task environments, planning objective, and
success criterion. Each WM is trained offline and used at
test time to score candidate action sequences through the terminal
latent goal-matching objective in \Cref{sec:method}. We use
receding-horizon CEM with $300$ candidates, $30$ refinement
iterations, and $30$ elites. ATLAS, LeWM, and DINO-WM are evaluated
on identical sampled episodes within each evaluation condition;
FAST-WM and PLDM use their released or saved episode draws from the
same task distributions and are evaluated with the same protocol and
episode-partitioning rule.

Planning success is defined as the fraction of episodes in which the
simulator reaches the target configuration. Unless stated otherwise,
the goal is selected $50$ steps ahead along the demonstration
trajectory. We evaluate $200$ episodes for each of five evaluation
seeds ($42$--$46$) and report the mean and standard deviation across
these seeds. For ATLAS, we use $\lambda{=}3.0$ for WEMReg and
$\mu{=}0.1$ for OOD-recovery across all tasks. Complete training,
evaluation, and planning hyperparameters are provided in
\Cref{app:setup}.
\subsection{Main results}
\label{sec:exp-main}

\Cref{fig:main} reports planning success at goal offset $50$. For each task and evaluation seed, we split the $200$ episodes into lower- and higher-novelty halves based on their distance from the training-state support. Specifically, each episode receives a novelty score given by the mean $k$-nearest-neighbor distance between states along the expert trajectory and a fixed bank of training states, and episodes are split at the median. We refer to the two subsets as in-distribution (ID) and OOD, respectively. Full details and exact values are provided in \Cref{app:idood,tab:main-idood} in the Appendix.

\begin{figure*}[t]
\centering
\vspace{-0.2inch}
\includegraphics[width=\textwidth]{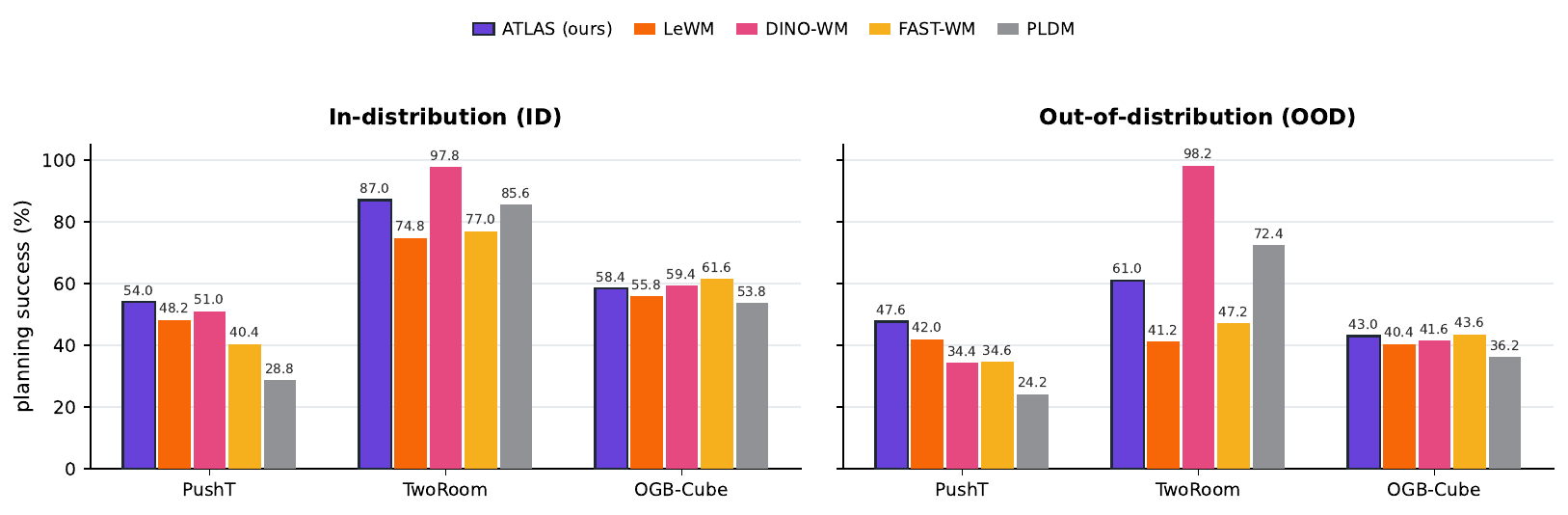}
\caption{\textbf{Planning success at goal offset $50$ on lower- and higher-novelty evaluation subsets.}
We refer to the two median-split subsets as ID and OOD, respectively. Bars show mean success over five evaluation seeds. The split is determined by a model-independent expert-path novelty score based on distance to training states. Exact values are reported in \Cref{tab:main-idood}.}
\label{fig:main}
\end{figure*}

Across all three tasks, ATLAS improves mean planning success over LeWM on both ID and OOD episodes. The gains are modest on PushT and OGBench-Cube and substantially larger on TwoRoom, where ATLAS improves from $74.8\%$ to $87.0\%$ on ID episodes and from $41.2\%$ to $61.0\%$ on OOD episodes. The largest improvement therefore occurs on higher-novelty TwoRoom episodes. Importantly, the gains are not restricted to OOD states: ATLAS also improves over LeWM on the lower-novelty subset in every task.
Performance relative to the broader baseline set is task dependent. ATLAS achieves the highest mean success on both novelty subsets in PushT, while several methods are similar on OGBench-Cube. On TwoRoom, DINO-WM performs substantially better, but it uses a frozen DINOv2-S/14 encoder pretrained on roughly $142$M external images, whereas ATLAS is trained from task data without external visual pretraining\citep{oquab2024dinov2learningrobustvisual}. We therefore treat DINO-WM as a reference for the benefit of a strong pretrained representation rather than as a controlled objective-level comparison. Other baselines, including PLDM and FAST-WM, are competitive on particular tasks, underscoring that the most direct comparison for isolating the effect of ATLAS is with LeWM and the controlled ablations in \Cref{sec:ablation}.

\subsection{Ablation study}
\label{sec:ablation}
We study the roles of marginal calibration and relational preservation with a $2\times2$ ablation over
{SIGReg, WEMReg} $\times$ {OOD-recovery off, on}. All four variants use the same architecture, data, training budget, and planner, and differ only in these objective terms. LeWM corresponds to SIGReg without OOD-recovery, while ATLAS uses WEMReg together with OOD-recovery. \Cref{tab:ablation} reports planning success at goal offset $50$.

The two objectives affect planning differently across tasks. On TwoRoom, either WEMReg or OOD-recovery alone produces a large improvement over LeWM, while the full ATLAS objective performs similarly to the stronger single-component variant. On PushT, OOD-recovery provides the larger individual gain, and the full model performs best. On OGBench-Cube, OOD-recovery again improves over LeWM, whereas WEMReg alone does not; the full objective recovers the strongest mean performance.
Overall, ATLAS attains the highest mean success on all three tasks, although its gain over the stronger single-component variant is small. These results suggest that WEMReg and OOD-recovery constrain different properties of the representation, but their downstream benefits need not be additive and can overlap in the representation bottlenecks they address. This is consistent with \Cref{prop:nonredundancy}, which establishes that neither objective determines the other without implying additive improvements in planning performance.

\begin{table}[t]
\caption{$2\times2$ ablation at goal offset $50$ over marginal calibration
(SIGReg vs.\ WEMReg) and relational preservation (OOD-recovery off vs.\ on).
Planning success rate (\%, mean $\pm$ std over 5 evaluation seeds);
highest mean per task in bold.}
\label{tab:ablation}
\begin{center}
\small
\setlength{\tabcolsep}{6pt}
\begin{tabular}{lccc}
\toprule
\textbf{Arm} & \textbf{PushT} & \textbf{TwoRoom} & \textbf{OGBench-Cube} \\
\midrule
LeWM (base)
& $45.10 \pm 3.65$ & $58.00 \pm 3.36$ & $48.10 \pm 2.94$ \\
ATLAS w/o WEMReg
& $50.30 \pm 3.85$ & $71.80 \pm 2.99$ & $50.50 \pm 2.55$ \\
ATLAS w/o OOD-recovery
& $47.00 \pm 2.00$ & $73.90 \pm 2.46$ & $46.90 \pm 2.60$ \\
ATLAS
& $\mathbf{50.80 \pm 0.68}$ & $\mathbf{74.00 \pm 1.92}$ & $\mathbf{50.70 \pm 2.69}$ \\
\bottomrule
\end{tabular}
\end{center}
\end{table}

\section{Analysis of latent geometry and planning reliability}
\label{sec:analysis}

The planner in \Cref{sec:method} selects actions by comparing predicted terminal latents with an encoded goal. Its behavior therefore depends on three properties highlighted by \Cref{thm:planning_stability}: preservation of state-to-state geometry, calibration of the latent distribution, and accuracy of multi-step latent prediction. We examine each property below. These analyses are intended as mechanism-oriented diagnostics rather than causal explanations of planning success.

\textbf{ATLAS preserves more OOD-related structure in the planning latent.}
We first ask whether the failure-associated novelty signal identified in \Cref{sec:intro} remains detectable in the representation used for planning. Using the frozen-model diagnostic from \Cref{app:diagnostic}, we compute a $k$-NN novelty score at different representation depths and measure how well it predicts downstream planning failure. In LeWM, this signal weakens substantially from the patch representation to the final planning latent. Under ATLAS, the planning latent retains much more of this signal, reaching $0.75$ compared with $0.77$ in the patch representation. Similar improvements are observed on TwoRoom and OGBench-Cube (\Cref{fig:ood-recover} in Appendix).

This behavior is consistent with the relational objective in \Cref{sec:method-ood}, which transfers normalized pairwise geometry from the patch representation to the planning latent without directly supervising OOD labels or novelty scores. Direct measurements of pairwise relational distortion across controlled variants are provided in \Cref{fig:reme}.
On PushT, where LeWM shows the largest patch-to-latent degradation, OOD-recovery also produces the larger single-component gain in the ablation study.

\begin{wrapfigure}{r}{0.47\textwidth}
\centering
\vspace{-0.8em}
\includegraphics[width=0.45\textwidth]{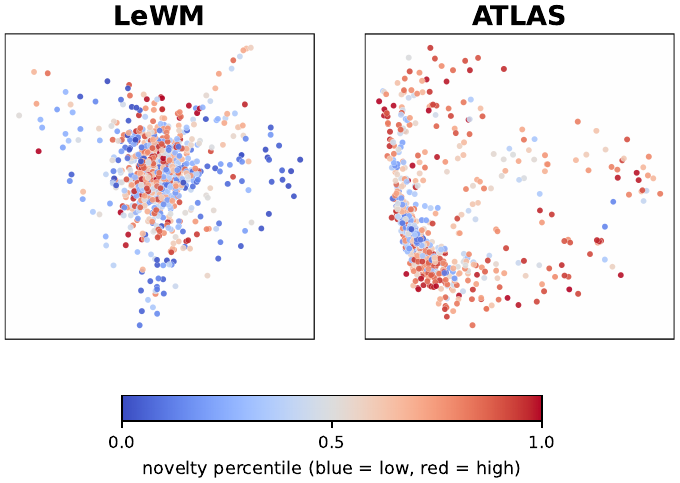}
\caption{\textbf{Planning-latent geometry (PushT).}
Two-dimensional PCA projections of $z$. Each point is a state, colored by its model-independent
true-state novelty percentile.}
\label{fig:ood-scatter}
\vspace{-1.0em}
\end{wrapfigure}

The same pattern is visible qualitatively in \Cref{fig:ood-scatter}. States are colored by model-independent true-state novelty after projecting the planning latent to two dimensions using PCA. On PushT, ATLAS shows the clearest structure: the in-distribution (blue, low-novelty) states sit close to one another, while the out-of-distribution (red, high-novelty) states lie farther from the rest, so distance in the latent tracks novelty; under LeWM this distance separation is much weaker. On TwoRoom and OGBench-Cube the effect is less pronounced, though the ATLAS projections still place high-novelty states farther from the in-distribution cluster than LeWM does (\Cref{fig:ood-scatter-appendix} in Appendix). Because these PCA plots are only two-dimensional projections, we use them as qualitative support for the stronger quantitative $k$-NN diagnostic rather than as a direct measure of relational preservation.

\begin{figure*}[b]
\centering
\vspace{-0.4inch}
\includegraphics[width=\textwidth]{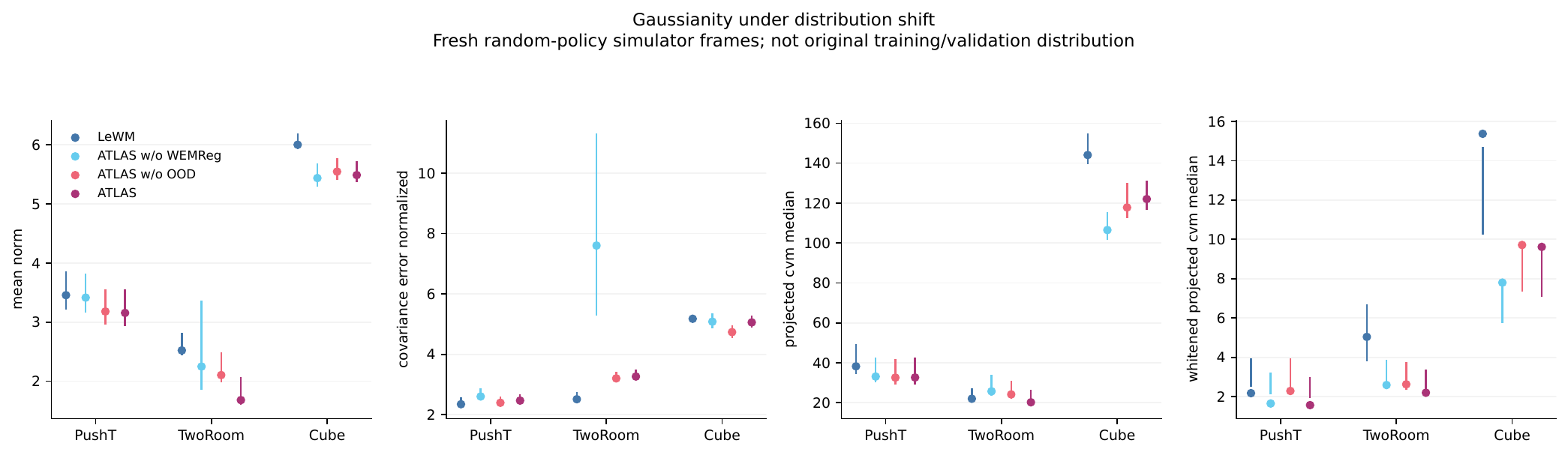}
\caption{\textbf{Latent-marginal diagnostics under distribution shift.}
Planning-latent diagnostics on fresh random-policy simulator frames.
From left to right, the panels report mean norm, normalized covariance
error, median projected CvM discrepancy, and median whitened projected
CvM discrepancy.}
\label{fig:oodshift}
\end{figure*}

\textbf{WEMReg improves latent-marginal calibration.}
Relational preservation and marginal calibration act on different properties of the representation. The former preserves normalized pairwise structure, while the latter controls the global scale and shape of the latent distribution. In \Cref{thm:planning_stability}, this distinction appears through the scale-mismatch term $\eta$.
We evaluate marginal calibration on fresh random-policy simulator frames, separate from the ID/OOD split used for planning evaluation. \Cref{fig:oodshift} reports mean deviation, covariance error, projected distributional discrepancy, and projected discrepancy after whitening. Across the evaluated tasks, variants using WEMReg improve several of these diagnostics, with especially clear reductions in covariance and projected-distribution errors on TwoRoom. The effect is not uniform across every task and statistic, but overall the results indicate improved latent-marginal calibration under WEMReg.

\textbf{ATLAS has lower multi-step latent prediction error.}
The third factor in \Cref{thm:planning_stability} is terminal prediction error. Because CEM evaluates candidate actions after repeatedly applying the learned dynamics, small one-step errors can accumulate over the rollout horizon. We therefore measure normalized latent prediction error as the horizon increases on PushT and TwoRoom (\Cref{fig:rollout}). ATLAS has lower error than LeWM at every measured horizon, with the gap generally increasing over longer rollouts. This shows that the full ATLAS model provides more accurate multi-step latent predictions on these shared-query evaluations.

\begin{wrapfigure}{r}{0.6\textwidth}
\centering
\vspace{-0.8em}
\includegraphics[width=\linewidth]{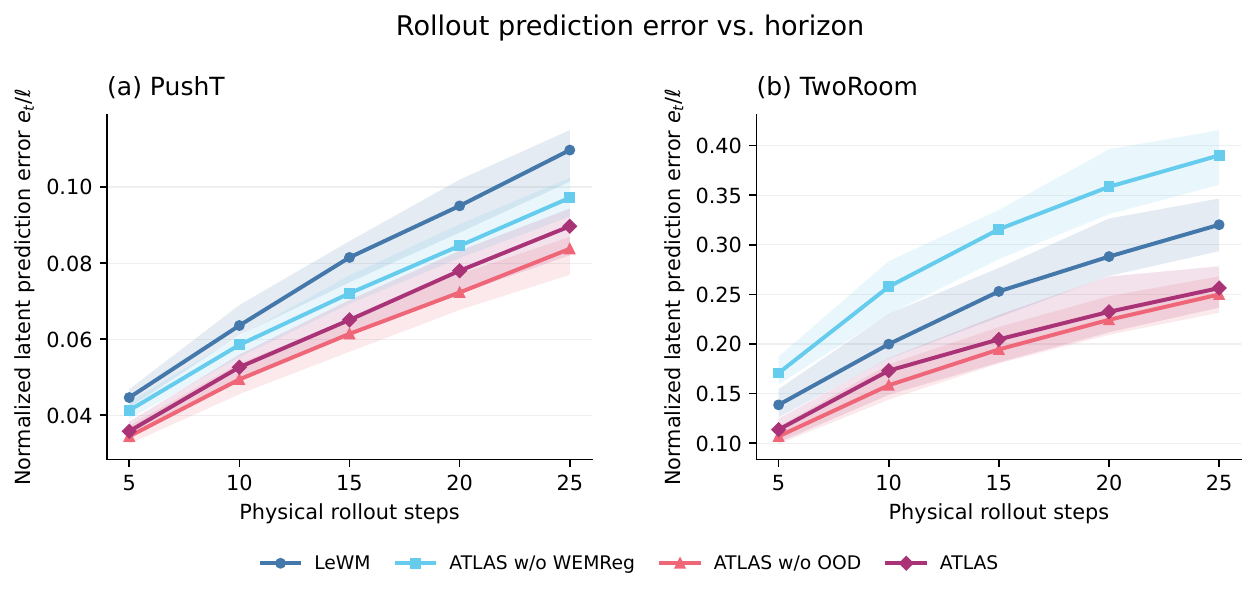}
\caption{\textbf{Multi-step rollout prediction error.}
Results on PushT (a) and TwoRoom (b) use $100$ shared queries and $128$
shared candidates per query. Lines show the across-query median of
candidate-median latent prediction errors, normalized by the fixed
checkpoint-specific reference-bank scale $\ell$. Shading denotes
$95\%$ percentile intervals from $1{,}000$ paired query-bootstrap
resamples. One prediction step spans five physical steps.}
\label{fig:rollout}
\vspace{-0.7em}
\end{wrapfigure}

Taken together, these diagnostics align with the three factors identified by the planning-stability analysis: OOD-recovery reduces relational distortion, WEMReg improves latent marginal calibration, and ATLAS exhibits lower multi-step prediction error.
Additional diagnostics of the planning-stability analysis, including the bound-versus-regret comparison and the decomposition of the prediction, relational, and scale terms in Theorem~\ref{thm:planning_stability}, are provided in Figures~\ref{fig:regret} and Figure~\ref{fig:theo3}.

\textbf{Closed-loop trajectories illustrate the resulting behavior.}
\Cref{fig:cases} shows a selected TwoRoom episode in which ATLAS succeeds while LeWM fails. ATLAS passes through the connecting opening and reaches the goal region, whereas LeWM remains on the opposite side of the wall. Analogous PushT and OGBench-Cube examples are shown in \Cref{fig:cases-appendix}. These trajectories are qualitative illustrations of the behavior underlying the aggregate success rates rather than representative estimates of typical behavior.

\begin{figure}[h]
\centering
\includegraphics[width=\linewidth]{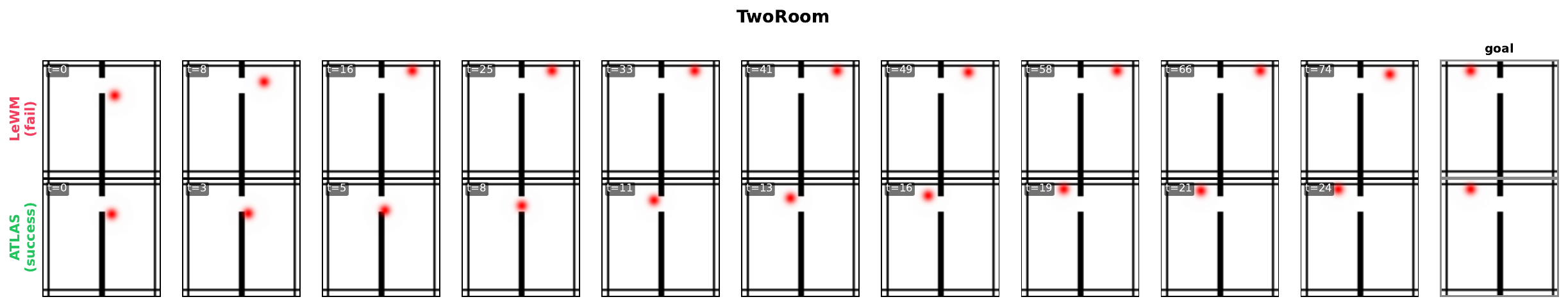}
\caption{\textbf{Selected closed-loop planning trajectory on TwoRoom.}
One episode in which LeWM fails and ATLAS succeeds (LeWM above ATLAS;
frame indices indicate elapsed steps, goal in the rightmost column). The
PushT and OGBench-Cube cases are shown in \Cref{fig:cases-appendix} in Appendix.}
\label{fig:cases}
\end{figure}

\section{Conclusion}
\label{sec:conclusion}

We introduced ATLAS, a training objective for latent WMs that jointly targets relational preservation and marginal calibration. OOD-recovery preserves normalized state-to-state geometry relative to an internal patch anchor, while WEMReg calibrates the global latent distribution. Our analysis further connects these representation properties to finite-candidate planning through relational distortion, latent-scale mismatch, and terminal prediction error.
Across PushT, TwoRoom, and OGBench-Cube, ATLAS improves mean goal-reaching success over LeWM on both lower- and higher-novelty evaluation subsets. Ablations show that relational preservation and marginal calibration provide distinct but partially overlapping benefits, while representation and rollout diagnostics show stronger novelty-related structure in the planning latent, improved marginal calibration, and lower multi-step prediction error. Together, these results support preserving planning-relevant latent geometry as a useful principle for improving the reliability of latent WM planning.

% \section*{AI Use Statement}
% Large language model based coding assistants were used by the authors during this
% project, in line with the ICLR policy on the use of large language models. Their use was
% limited to software engineering and presentation support: assisting with data-processing
% and plotting code, refactoring evaluation scripts, drafting and typesetting figures and
% tables, and drafting and copy-editing portions of the manuscript text. They were not used to generate research
% ideas, design the method, or produce experimental results. All experiments were run and
% verified by the authors, and the authors checked every reported number, claim, and figure
% against the underlying data and take full responsibility for the content of the paper.

\bibliography{iclr2027_conference}
\bibliographystyle{iclr2027_conference}

\appendix
% after \appendix, section/subsection labels are appendices -> cleveref says "Appendix X"
\crefalias{section}{appendix}
\crefalias{subsection}{appendix}
\section{Experimental setup details}
\label{app:setup}

\Cref{tab:hparams} lists the training, model, and planning hyperparameters. All world models use
the same architecture, the arms differ only in the anti-collapse
regularizer (SIGReg vs.\ WEMReg) and the OOD-recovery weight $\mu$. Values are shared across the
three tasks unless noted.

\begin{table}[h]
\caption{Training, model, and planning hyperparameters.}
\label{tab:hparams}
\begin{center}
\small
\begin{tabular}{lll}
\toprule
\bf Group & \bf Hyperparameter & \bf Value \\
\midrule
\multirow{5}{*}{Model}
 & Latent / embedding dimension & $192$ \\
 & Predictor & 6-layer causal Transformer \\
 & Attention heads / MLP dimension & $16$ / $2048$ \\
 & History length & $3$ frames \\
 & Frameskip & $5$ \\
\addlinespace
\multirow{6}{*}{Training}
 & Epoch budget (all tasks) & $10$ \\
 & Batch size & $128$ \\
 & Optimizer & AdamW \\
 & Learning rate (cosine schedule) & $5\times10^{-5}$ \\
 & Weight decay & $1\times10^{-3}$ \\
 & Training seed & $3072$ \\
\addlinespace
\multirow{4}{*}{Objective}
 & WEMReg weight $\lambda$ & $3.0$ \\
 & WEMReg projections & $1024$ \\
 & OOD-recovery weight $\mu$ & $0.1$ (ours); $0$ (LeWM) \\
 & Prediction loss & next-step latent MSE \\
\addlinespace
\multirow{6}{*}{Planning (CEM)}
 & Candidates / iterations / elites & $300$ / $30$ / $30$ \\
 & Sampling variance scale & $1.0$ \\
 & Planning horizon & $5$ \\
 & Receding horizon (replan every) & $5$ \\
 & Evaluation budget & $\text{offset}+25$ steps \\
 & Episodes $\times$ seeds & $200 \times 5$ (seeds $42$--$46$) \\
\bottomrule
\end{tabular}
\end{center}
\end{table}

\paragraph{Datasets.} All models are trained on the offline demonstration datasets shipped with the
\textsc{stable-worldmodel} harness: \textbf{TwoRoom} and \textbf{PushT} (2D navigation and
manipulation) and \textbf{OGBench-Cube} (3D robot-arm manipulation, from OGBench~\citep{ogbench}).
Each dataset is split $90/10$ into training and validation clips, with the validation set used only
to monitor training. Evaluation goals are sampled from the full set of valid episodes,
not from a held-out test split, and every method is evaluated on $200$ episodes per
seed at goal offset $50$.

\subsection{In-distribution / out-of-distribution split}
\label{app:idood}
We split each evaluation condition (task, seed) into in-distribution (ID) and out-of-distribution
(OOD) halves by how far each episode's expert path lies from the training distribution, using a
fixed, model-independent rule. We draw a bank of $50{,}000$ training states, standardize each
coordinate by its bank mean and standard deviation, and fit a $k$-nearest-neighbor index with
$k{=}50$. For each of the $200$ episodes we average, over the states of its expert path (from the
start row through the goal), the distance to the $k$ nearest bank states, giving one OOD score per
episode; episodes above the per-condition median are labeled OOD and the rest ID ($100$ each). The
bank, standardization, and neighbor count are fixed and use no model, so the labels depend only on
the states.

\begin{table}[h]
\caption{Main results at goal offset $50$, split into in-distribution (ID) and out-of-distribution
(OOD) episodes (\%, mean $\pm$ std over 5 seeds). Best per column in bold; exact data of
\Cref{fig:main}.}
\label{tab:main-idood}
\begin{center}
\setlength{\tabcolsep}{4pt}
\resizebox{\textwidth}{!}{%
\begin{tabular}{lcccccc}
\toprule
 & \multicolumn{2}{c}{\bf PushT} & \multicolumn{2}{c}{\bf TwoRoom} & \multicolumn{2}{c}{\bf OGBench-Cube}\\
\cmidrule(lr){2-3}\cmidrule(lr){4-5}\cmidrule(lr){6-7}
\bf Method & \bf ID & \bf OOD & \bf ID & \bf OOD & \bf ID & \bf OOD \\
\midrule
 ATLAS (ours) & $\mathbf{54.00 \pm 3.63}$ & $\mathbf{47.60 \pm 2.58}$ & $87.00 \pm 2.28$ & $61.00 \pm 2.19$ & $58.40 \pm 3.72$ & $43.00 \pm 4.94$ \\
 LeWM & $48.20 \pm 4.62$ & $42.00 \pm 5.76$ & $74.80 \pm 2.79$ & $41.20 \pm 4.35$ & $55.80 \pm 4.31$ & $40.40 \pm 4.22$ \\
 DINO-WM & $51.00 \pm 5.93$ & $34.40 \pm 4.67$ & $\mathbf{97.80 \pm 1.17}$ & $\mathbf{98.20 \pm 0.98}$ & $59.40 \pm 4.50$ & $41.60 \pm 7.28$ \\
 FAST-WM & $40.40 \pm 5.08$ & $34.60 \pm 4.96$ & $77.00 \pm 4.00$ & $47.20 \pm 3.87$ & $\mathbf{61.60 \pm 4.63}$ & $\mathbf{43.60 \pm 5.71}$ \\
 PLDM & $28.80 \pm 4.87$ & $24.20 \pm 5.19$ & $85.60 \pm 3.72$ & $72.40 \pm 1.85$ & $53.80 \pm 4.62$ & $36.20 \pm 6.97$ \\
\bottomrule
\end{tabular}}
\end{center}
\end{table}

\section{OOD diagnostic: protocol and full results}
\label{app:diagnostic}

This appendix details the probe used in \Cref{sec:intro} (and reused in
\Cref{sec:analysis}) and reports its per-representation results at goal offset $50$.

\paragraph{Representations probed.}
For a frozen model we score, for the same input frame, a set of representations drawn
from the encoder at increasing depth: the \texttt{[CLS]} token at intermediate ViT
blocks (\emph{L3, L6, L9}), the mean-pooled final-layer patch tokens (\emph{patch}), the
final \texttt{[CLS]} token before the projector (\emph{cls}), and the projected
\texttt{[CLS]} token that SIGReg regularizes and that the planner uses, i.e.\ the
planning latent $z$. As a model-independent reference we also score an \emph{oracle}
representation built from the true environment state.

\paragraph{OOD score.}
Each representation is scored by a $k$-NN distance to a bank of $N{=}1.5\text{k}$ training
frames encoded once and standardized ($k{=}50$). For a query frame the score is the mean
distance to its $k$ nearest bank neighbors; a large distance means the frame is far from
training support, i.e.\ OOD. The oracle uses the same $k$-NN construction on the raw
true-state.

\paragraph{Metric: does the OOD score predict planning failure?}
We run the \emph{unchanged} official planner (latent CEM) on $n{=}200$ evaluation
episodes and log, at every step, each representation's OOD score; planning itself is
untouched, only the logged signal differs. Per episode we aggregate the score over its
last three steps and label the episode by whether the real simulator marks it a failure,
and summarize each representation by $\mathrm{AUROC}(\text{OOD score}\rightarrow
\text{failure})$ (chance $=0.5$). The true-state oracle, computed from the ground-truth
state, measures the \emph{true} degree to which a model's failures are OOD-driven. Our aim
is therefore not to raise a representation's AUROC as high as possible, but to make the
planning latent's AUROC \emph{agree} with the oracle's: a latent scoring far below the
oracle is blind to OOD structure its failures actually carry, whereas one that matches the
oracle reflects that structure as faithfully as the true state does.

\paragraph{Results.}
\Cref{tab:probe-full} reports this AUROC for every representation on the three tasks, at
goal offset $50$, for the baseline LeWM model, and \Cref{tab:probe-atlas} reports
the same for ATLAS; both run on the same episode set under the same planner as our main
results ($n{=}200$ episodes). For LeWM the OOD signal is present in the mid/late ViT features
and the patch tokens (often $0.6$--$0.9$) but drops sharply at the SIGReg-regularized
planning latent $z$, to $0.44$ on PushT (below chance) and to $0.55$/$0.49$ on TwoRoom/Cube,
in every case well below the patch anchor and the true-state oracle. This is the washout that
the OOD-recovery term of \Cref{sec:method} repairs, as \Cref{sec:analysis} shows: under ATLAS
$z$ rises to $0.75$/$0.64$/$0.69$, matching its own patch anchor and oracle.

\begin{table}[h]
\caption{OOD diagnostic at goal offset $50$ ($n{=}200$ episodes, baseline LeWM). $\mathrm{AUROC}(\text{OOD score}\rightarrow\text{planning failure})$ per
representation (shallow $\rightarrow$ deep, ending at the planning latent $z$). ``oracle'' is
the true-state reference. Chance $=0.5$.}
\label{tab:probe-full}
\begin{center}
\small
\begin{tabular}{lcccccc|c}
\bf Task & \bf L3 & \bf L6 & \bf L9 & \bf patch & \bf cls & \bf $z$ (SIGReg) & \bf oracle\\
\hline \\[-1.8ex]
PushT        & 0.59 & 0.81 & 0.83 & 0.76 & 0.72 & 0.44 & 0.88 \\
TwoRoom      & 0.69 & 0.70 & 0.60 & 0.79 & 0.71 & 0.55 & 0.69 \\
OGBench-Cube & 0.82 & 0.86 & 0.85 & 0.71 & 0.79 & 0.49 & 0.76 \\
\end{tabular}
\end{center}
\end{table}

\begin{table}[h]
\caption{OOD diagnostic at goal offset $50$ for \textbf{ATLAS} ($n{=}200$ episodes, same
episode set and planner as \Cref{tab:probe-full}).
$\mathrm{AUROC}(\text{OOD score}\rightarrow\text{planning failure})$ per representation.}
\label{tab:probe-atlas}
\begin{center}
\small
\begin{tabular}{lcccccc|c}
\bf Task & \bf L3 & \bf L6 & \bf L9 & \bf patch & \bf cls & \bf $z$ & \bf oracle\\
\hline \\[-1.8ex]
PushT        & 0.55 & 0.63 & 0.74 & 0.77 & 0.77 & 0.75 & 0.82 \\
TwoRoom      & 0.30 & 0.38 & 0.61 & 0.65 & 0.69 & 0.64 & 0.59 \\
OGBench-Cube & 0.82 & 0.89 & 0.88 & 0.69 & 0.84 & 0.69 & 0.70 \\
\end{tabular}
\end{center}
\end{table}

\begin{figure}[h]
\centering
\includegraphics[width=0.62\textwidth]{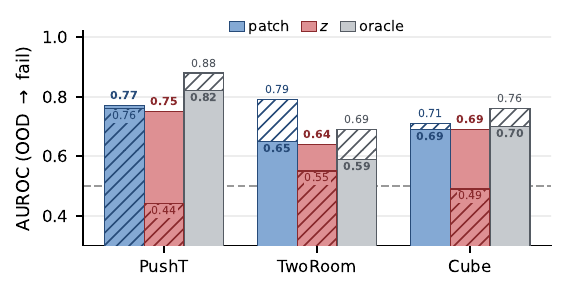}
\caption{\textbf{OOD-related diagnostic signal across representations.}
AUROC for predicting planning failure from a $k$-NN novelty score at goal offset $50$.
Hatched bars denote LeWM and solid bars ATLAS. Scores are computed from mean-pooled patch
features, the planning latent $z$, and true environment states (oracle reference). This is
a visualization of the per-representation values in \Cref{tab:probe-full,tab:probe-atlas}.}
\label{fig:ood-recover}
\end{figure}

\section{Proofs for the Theoretical Analysis}
\label{app:theory}

\subsection{Definitions and Setup}
\label{app:theory_details}

\paragraph{Distributional objectives.}
Let $\mathcal P_2(\R^d)$ denote distributions with finite second moments,
$\gamma_d=\mathcal N(\vzero,\mI_d)$, and $P_{\vu}$ the distribution of
$\vu^\top\rvz$ for $\rvz\sim P$ and $\vu\in\mathbb S^{d-1}$.
Here $W_2$ is the quadratic Wasserstein distance, and WEMReg is the
exact directional expectation in Eq.~\eqref{eq:theory_population_sot}.
For $\mZ=[\vz_1^\top;\ldots;\vz_N^\top]$, $N\geq2$, set
$P_{\mZ}=N^{-1}\sum_i\delta_{\vz_i}$ and
$\mathcal L_{\mathrm{WEMReg}}(\mZ)=\mathcal L_{\mathrm{WEMReg}}(P_{\mZ})$.
Its exact finite-batch decomposition is given in
Lemma~\ref{lem:quantile_identity}.
For finitely many frequencies $t_j\in\R$ and weights $\omega_j>0$, define
\begin{equation}
\begin{aligned}
\mathcal L_{\mathrm{CF}}(P)
&=\E_{\rvu\sim\operatorname{Unif}(\mathbb S^{d-1})}
  \sum_{j=1}^J\omega_j
  \left|\varphi_{P,\rvu}(t_j)-e^{-t_j^2/2}\right|^2,\\
\varphi_{P,\vu}(t)&=\E_{\rvz\sim P}[e^{\mathrm i t\vu^\top\rvz}].
\end{aligned}
\label{eq:theory_cf}
\end{equation}
This includes finite-frequency Epps--Pulley quadrature up to a positive
normalization factor \citep{balestriero2025lejepa,lewm}.
Theorem~\ref{thm:cf_ambiguity} concerns this finite-frequency objective,
not integration over all frequencies.

\paragraph{Relational geometry.}
Let $\bar{\vp}_i\in\R^m$ be the mean-pooled patch-token anchors and
$\mP=[\bar{\vp}_1^\top;\ldots;\bar{\vp}_N^\top]$, held fixed when
differentiating the relational loss.
For $\mX=[\vx_1^\top;\ldots;\vx_N^\top]$, let
$\vmu_{\mX}$ and $\boldsymbol\sigma_{\mX}$ be the empirical mean and
coordinate-wise standard deviations, with
$\vmu_{\mX}=N^{-1}\sum_i\vx_i$ and
$(\boldsymbol\sigma_{\mX})_j^2
=N^{-1}\sum_i((\vx_i)_j-(\vmu_{\mX})_j)^2$.
For $\varepsilon_0\geq0$, set
\begin{equation}
\begin{aligned}
\mS_{\mX}&=\operatorname{diag}(\boldsymbol\sigma_{\mX}+\varepsilon_0\vone),\\
c_{\mX}&=\frac1{N^2}\sum_{i,j}\|\mS_{\mX}^{-1}(\vx_i-\vx_j)\|_2,
\qquad
d_{\mX}(i,j)=\frac{\|\mS_{\mX}^{-1}(\vx_i-\vx_j)\|_2}{c_{\mX}}.
\end{aligned}
\label{eq:theory_relational_metric}
\end{equation}
Here $\vone$ is the all-ones vector. Assume positive diagonal entries
of $\mS_{\mX}$ and $c_{\mX}>0$.
The relational loss is
\begin{equation}
\mathcal L_{\mathrm{OOD}}(\mZ,\mP)
=\frac1{N^2}\sum_{i,j}
  \bigl(d_{\mZ}(i,j)-d_{\mP}(i,j)\bigr)^2.
\label{eq:theory_relational_loss}
\end{equation}

\paragraph{Evaluation and planning.}
Fix a finite pool $\sI=\{1,\ldots,N\}$ and normalizers shared across all
queries and candidates, computed on this pool or imported from a fixed
reference bank. Define
\begin{equation}
\epsilon_{\mathrm{rel}}
=\max_{i,j\in\sI}|d_{\mZ}(i,j)-d_{\mP}(i,j)|.
\label{eq:theory_uniform_rel}
\end{equation}
When the relational loss uses this same pool and these same normalizers,
the maximum squared error is bounded by the sum, giving
\begin{equation}
\epsilon_{\mathrm{rel}}
\leq N\sqrt{\mathcal L_{\mathrm{OOD}}(\mZ,\mP)}.
\label{eq:theory_mse_to_uniform}
\end{equation}
For $q\in\sI$, a nonempty bank $\sB\subseteq\sI$, and
$1\leq k\leq|\sB|$, the mean $k$NN score is
\begin{equation}
s_{\mX}(q)=\min_{\substack{\sJ\subseteq\sB\\|\sJ|=k}}
\frac1k\sum_{j\in\sJ}d_{\mX}(q,j),
\label{eq:theory_knn}
\end{equation}
which is independent of tie breaking.
For a nonempty candidate set $\sC\subseteq\sI$ and goal $g\in\sI$, let
$\vz_i$ and $\widehat{\vz}_i$ be the true and predicted terminal
embeddings, respectively, and set
\begin{equation}
\begin{aligned}
e&=\max_{i\in\sC}\|\widehat{\vz}_i-\vz_i\|_2,
&M&=\max_{i\in\sC}d_{\mP}(i,g),\\
\eta&=\frac{\max_j|(\boldsymbol\sigma_{\mZ})_j-1|}{1+\varepsilon_0},
&\ell&=(1+\varepsilon_0)c_{\mZ}.
\end{aligned}
\label{eq:theory_planning_constants}
\end{equation}

\paragraph{Scope.}
The guarantees concern the evaluated pool and finite candidate set;
small training losses alone do not control errors on unseen states.
The objectives need not change $e$, $\epsilon_{\mathrm{rel}}$, and $\eta$
independently or monotonically, and exact Gaussian matching and anchor
geometry preservation need not be simultaneously feasible.
Task-cost guarantees additionally require anchor alignment
(Corollary~\ref{cor:task_cost}).

\subsection{Exact Finite-Sample Sliced Transport}
\label{app:quantile_identity}

\begin{lemma}[Gaussian quantile-cell decomposition]
\label{lem:quantile_identity}
Let $N\geq2$, $x_{(1)}\leq\cdots\leq x_{(N)}$, and
$\widehat P=N^{-1}\sum_i\delta_{x_{(i)}}$.
With $\Phi$ the standard Gaussian cumulative distribution function, set
\begin{equation}
m_i=N\int_{(i-1)/N}^{i/N}\Phi^{-1}(p)\,dp,
\qquad v_N=\frac1N\sum_i m_i^2,
\qquad\kappa_N=1-v_N.
\label{eq:theory_cell_constants}
\end{equation}
Then
\begin{equation}
W_2^2(\widehat P,\gamma_1)
=\kappa_N+\frac1N\sum_i(x_{(i)}-m_i)^2.
\label{eq:theory_exact_sot}
\end{equation}
Moreover, $N^{-1}\sum_i m_i=0$, $0<v_N<1$, and $\kappa_N>0$.
\end{lemma}

\begin{proof}
Write $q=\Phi^{-1}$ and $I_i=((i-1)/N,i/N]$.
Since $\int_{I_i}(q(p)-m_i)\,dp=0$, the quantile representation gives
\[
\begin{aligned}
W_2^2(\widehat P,\gamma_1)
&=\sum_i\int_{I_i}(x_{(i)}-q(p))^2\,dp\\
&=\frac1N\sum_i(x_{(i)}-m_i)^2
  +\sum_i\int_{I_i}(q(p)-m_i)^2\,dp.
\end{aligned}
\]
The final sum equals $\int_0^1q(p)^2\,dp-v_N=1-v_N=\kappa_N$,
while $N^{-1}\sum_i m_i=\int_0^1q(p)\,dp=0$.
As $q$ is nonconstant on every cell, $\kappa_N>0$;
as $m_N>0$ for $N\geq2$, $v_N>0$. Thus $0<v_N<1$.
\end{proof}

Applying the lemma to sorted projections gives the exact batch WEMReg
cost, with non-optimizable residual $\kappa_N$.
This lower bound is attained by $x_{(i)}=m_i$ in one projection, but
need not be simultaneously attainable in all directions of a finite
high-dimensional point cloud.

\subsection{Proof of Finite-Frequency Ambiguity}
\label{app:cf_ambiguity}

\begin{proof}[Proof of Theorem~\ref{thm:cf_ambiguity}]
Choose $J+3$ disjoint spherical shells of positive $\gamma_d$-measure.
Their indicators span a $(J+3)$-dimensional space of bounded radial
functions. The following $J+2$ homogeneous linear constraints therefore
admit a nonzero solution $r$, normalized so that $\|r\|_\infty=1$:
\begin{equation}
\begin{aligned}
\int r\,d\gamma_d&=0,\qquad
\int\|\vx\|_2^2r(\vx)\,d\gamma_d(\vx)=0,\\
\int\cos\!\bigl(t_j(\vx)_1\bigr)r(\vx)\,d\gamma_d(\vx)&=0,
\qquad j=1,\ldots,J.
\end{aligned}
\label{eq:app_radial_constraints}
\end{equation}
Define $dP=(1+r/2)\,d\gamma_d$.
The first constraint makes $P$ a probability measure with
$1/2\leq dP/d\gamma_d\leq3/2$, hence finite second moments.
Also, $P\neq\gamma_d$ because $r$ is nonzero on a set of positive measure.
Orthogonal invariance and the second constraint give
\[
\E_P[\rvz]=\vzero,\qquad
\operatorname{Cov}_P(\rvz)
=\frac{\E_P\|\rvz\|_2^2}{d}\mI_d=\mI_d.
\]
Reflection symmetry eliminates the imaginary parts of the characteristic
functions. The remaining constraints and orthogonal invariance give
$\varphi_{P,\vu}(t_j)=e^{-t_j^2/2}$ for every unit $\vu$ and every $j$,
so $\mathcal L_{\mathrm{CF}}(P)=0$.

Finally, orthogonal invariance makes $W_2^2(P_{\vu},\gamma_1)$ independent
of $\vu$. If this value were zero, all unit projections would be standard
Gaussian, so $\E_P[e^{\mathrm i\vv^\top\rvz}]=e^{-\|\vv\|_2^2/2}$
for every $\vv\in\R^d$. Uniqueness of characteristic functions would
imply $P=\gamma_d$, a contradiction. Thus
$\mathcal L_{\mathrm{WEMReg}}(P)>0$.
\end{proof}

\subsection{WEMReg Controls the Normalization Mismatch}
\label{app:sot_calibration}

\begin{lemma}[Quantitative moment calibration]
\label{lem:sot_calibration}
Let $P\in\mathcal P_2(\R^d)$ have mean $\vmu$ and covariance $\mSigma$,
and set $s=\mathcal L_{\mathrm{WEMReg}}(P)$. Then
$\|\vmu\|_2^2\leq ds$ and
\begin{equation}
\|\mSigma-\mI_d\|_{\mathrm F}\leq C_d(s),
\label{eq:app_covariance_control}
\end{equation}
where $\|\cdot\|_{\mathrm F}$ is the Frobenius norm and
\begin{equation}
C_d(s)=\sqrt{\frac{d(d+2)}2}
\left[\sqrt d(1+\sqrt s)+1\right]\sqrt s.
\label{eq:app_calibration_constant}
\end{equation}
Consequently, $\sigma_j=\sqrt{\mSigma_{j,j}}$ satisfies
$\max_j|\sigma_j-1|\leq C_d(s)$. For the reference distribution used
in relational standardization, this gives
\begin{equation}
\eta\leq\frac{C_d(s)}{1+\varepsilon_0}.
\label{eq:app_eta_control}
\end{equation}
\end{lemma}

\begin{proof}
Let $\rvu\sim\operatorname{Unif}(\mathbb S^{d-1})$ and
$q(\vu)=\sqrt{\vu^\top\mSigma\vu}$.
Centering any coupling and applying the reverse triangle inequality
in $L^2$ gives
\begin{equation}
W_2^2(P_{\vu},\gamma_1)
\geq(\vu^\top\vmu)^2+(q(\vu)-1)^2.
\label{eq:app_projected_moment_bound}
\end{equation}
Since $\E[\rvu\rvu^\top]=\mI_d/d$, averaging proves
$\|\vmu\|_2^2/d+\E(q(\rvu)-1)^2\leq s$. Hence
\begin{equation}
\sqrt{\frac{\operatorname{tr}(\mSigma)}d}
=\|q\|_{L^2}\leq1+\sqrt s.
\label{eq:app_trace_bound}
\end{equation}
For any symmetric $\mA$, the spherical fourth-moment identity is
\begin{equation}
\E(\rvu^\top\mA\rvu)^2
=\frac{\operatorname{tr}(\mA)^2+2\|\mA\|_{\mathrm F}^2}{d(d+2)}.
\label{eq:app_spherical_fourth}
\end{equation}
Taking $\mA=\mSigma-\mI_d$ and using
$q(\vu)\leq\sqrt{\operatorname{tr}(\mSigma)}$ yields
\begin{equation}
\begin{aligned}
\E(\rvu^\top\mA\rvu)^2
&=\E\bigl[(q(\rvu)-1)^2(q(\rvu)+1)^2\bigr]\\
&\leq\left[\sqrt d(1+\sqrt s)+1\right]^2s.
\end{aligned}
\label{eq:app_directional_covariance}
\end{equation}
Dropping the nonnegative trace term in
Eq.~\eqref{eq:app_spherical_fourth} proves the covariance bound.
Finally,
\[
|\sigma_j-1|
=\frac{|\mSigma_{j,j}-1|}{\sigma_j+1}
\leq|\mSigma_{j,j}-1|
\leq\|\mSigma-\mI_d\|_{\mathrm F},
\]
which also proves the bound on $\eta$.
\end{proof}

The lemma applies to population or empirical reference distributions.
Its $s$ is the exact directional expectation; substituting a finite
Monte Carlo estimate requires an additional integration-error bound.

\subsection{Proof of Non-Redundancy}
\label{app:nonredundancy}

\begin{proof}[Proof of Proposition~\ref{prop:nonredundancy}]
Take $N=3$ and $d=m=1$.
Gaussian symmetry gives $(m_1,m_2,m_3)=(-b,0,b)$ with $b>0$.
Set $\mP=(-b,0,b)^\top$, $\mZ^{\mathrm p}=(-b,b,0)^\top$, and
$\mZ^{(c)}=c\mP$ for $c>0$.
In one dimension, the positive standardization factor cancels, so
\[
d_{\mX}(i,j)=\frac{9|x_i-x_j|}{\sum_{a,b}|x_a-x_b|}.
\]
The distances for pairs $(1,2),(1,3),(2,3)$ are respectively
$\frac98(1,2,1)$ for $\mP$ and $\frac98(2,1,1)$ for
$\mZ^{\mathrm p}$. Together with their identical sorted values and
Lemma~\ref{lem:quantile_identity}, this gives
\begin{equation}
\begin{aligned}
\mathcal L_{\mathrm{WEMReg}}(\mP)
&=\mathcal L_{\mathrm{WEMReg}}(\mZ^{\mathrm p})=\kappa_3,\\
\mathcal L_{\mathrm{OOD}}(\mP,\mP)&=0,\qquad
\mathcal L_{\mathrm{OOD}}(\mZ^{\mathrm p},\mP)
=\frac49\left(\frac98\right)^2=\frac9{16}.
\end{aligned}
\label{eq:theory_permutation_example}
\end{equation}
The two one-dimensional projection directions have equal transport
costs by Gaussian symmetry. Positive scaling preserves all normalized
distances, while the sorted scaled values are $cm_i$. Thus
\begin{equation}
\begin{aligned}
\mathcal L_{\mathrm{OOD}}(\mZ^{(c)},\mP)&=0,\\
\mathcal L_{\mathrm{WEMReg}}(\mZ^{(c)})
&=\kappa_3+\frac13\sum_i(cm_i-m_i)^2
=\kappa_3+(c-1)^2v_3.
\end{aligned}
\label{eq:theory_scale_example}
\end{equation}
Taking $c\neq1$ proves the second claim. The construction is valid
for every $\varepsilon_0\geq0$ because all scalar standardization
denominators are positive.
\end{proof}

\subsection{Stability of Finite Minimization}
\label{app:finite_minimization}

\begin{lemma}[Uniform cost perturbations]
\label{lem:finite_minimum}
Let $\sC$ be nonempty and finite and let $f,g:\sC\to\R$ satisfy
$\max_{i\in\sC}|f(i)-g(i)|\leq\epsilon$ for $\epsilon\geq0$.
Then
\[
\left|\min_{\sC}f-\min_{\sC}g\right|\leq\epsilon,
\qquad
f(\widehat i)-\min_{\sC}f\leq2\epsilon
\quad(\widehat i\in\arg\min_{\sC}g).
\]
If $f$ has a unique minimizer $i^\star$ with
$f(i)-f(i^\star)>2\epsilon$ for every $i\neq i^\star$, then every
minimizer of $g$ equals $i^\star$.
\end{lemma}

\begin{proof}
Taking minima in $f-\epsilon\leq g\leq f+\epsilon$ proves the first
claim. For $i^\star\in\arg\min f$ and $\widehat i\in\arg\min g$,
\[
f(\widehat i)\leq g(\widehat i)+\epsilon
\leq g(i^\star)+\epsilon\leq f(i^\star)+2\epsilon.
\]
The strict margin rules out $\widehat i\neq i^\star$.
\end{proof}

\subsection{Proof of Reliability and Planning Stability}
\label{app:planning_stability}

\begin{proof}[Proof of Theorem~\ref{thm:planning_stability}]
For every $k$-element subset $\sJ\subseteq\sB$, its mean distances
under $d_{\mZ}$ and $d_{\mP}$ differ by at most
$\epsilon_{\mathrm{rel}}$.
Applying the first part of Lemma~\ref{lem:finite_minimum} to the finite
family of these subsets proves Eq.~\eqref{eq:theory_knn_stability}.

For planning, put $a=1+\varepsilon_0$ and
$\vv_i=\mS_{\mZ}^{-1}(\vz_i-\vz_g)$.
By definition, $\|\mS_{\mZ}/a-\mI_d\|_{\mathrm{op}}=\eta$, so
\begin{equation}
\begin{aligned}
\left|\frac{\|\vz_i-\vz_g\|_2}{\ell}-d_{\mZ}(i,g)\right|
&=\frac1{c_{\mZ}}
\left|\|(\mS_{\mZ}/a)\vv_i\|_2-\|\vv_i\|_2\right|\\
&\leq\frac{\|(\mS_{\mZ}/a-\mI_d)\vv_i\|_2}{c_{\mZ}}
\leq\eta\,d_{\mZ}(i,g).
\end{aligned}
\label{eq:app_raw_normalized}
\end{equation}
Combining this with the reverse triangle inequality and
$d_{\mZ}(i,g)\leq M+\epsilon_{\mathrm{rel}}$ gives, uniformly on $\sC$,
\begin{equation}
\begin{aligned}
\left|\frac{\|\widehat{\vz}_i-\vz_g\|_2}{\ell}-d_{\mP}(i,g)\right|
&\leq\frac e\ell+\eta d_{\mZ}(i,g)+\epsilon_{\mathrm{rel}}\\
&\leq\frac e\ell+(1+\eta)\epsilon_{\mathrm{rel}}+\eta M
=\epsilon_{\mathrm{plan}}.
\end{aligned}
\label{eq:app_uniform_planning_cost}
\end{equation}
Since $\ell>0$, minimizing $\|\widehat{\vz}_i-\vz_g\|_2^2$ is equivalent
to minimizing $h(i)=\|\widehat{\vz}_i-\vz_g\|_2/\ell$.
Lemma~\ref{lem:finite_minimum}, applied to $f(i)=d_{\mP}(i,g)$ and
$h(i)$, proves the regret bound.
\end{proof}

\paragraph{Margins and exact preservation.}
Under the assumptions of Theorem~\ref{thm:planning_stability}, including
$\eta<1$, a unique anchor minimizer $i^\star$ is guaranteed to be
selected whenever
\begin{equation}
d_{\mP}(i,g)-d_{\mP}(i^\star,g)>2\epsilon_{\mathrm{plan}}
\qquad(i\in\sC\setminus\{i^\star\}),
\label{eq:theory_planning_margin}
\end{equation}
by Lemma~\ref{lem:finite_minimum}; this condition is sufficient, not
necessary. Similarly, the kNN bound implies
\[
s_{\mP}(q_1)-s_{\mP}(q_2)>2\epsilon_{\mathrm{rel}}
\quad\Longrightarrow\quad s_{\mZ}(q_1)>s_{\mZ}(q_2),
\]
without requiring identical nearest-neighbor sets.
If $\epsilon_{\mathrm{rel}}=0$, all kNN scores agree on the pool.
If additionally $e=\eta=0$, every selected candidate is anchor-optimal,
and equals the anchor minimizer when it is unique.

\subsection{From One-Step Prediction to Terminal Error}
\label{app:rollout_error}

\begin{lemma}[Finite-horizon prediction error]
\label{lem:rollout_error}
Let $g:\R^d\times\sA\to\R^d$ be the learned predictor.
For each candidate $i$, let $\vz_{i,0},\ldots,\vz_{i,T}$ be its true
encoded trajectory, and define
$\widehat{\vz}_{i,0}=\vz_{i,0}$ and
$\widehat{\vz}_{i,t+1}=g(\widehat{\vz}_{i,t},\va_{i,t})$.
Assume, for every relevant candidate and time,
\[
\|g(\vz_{i,t},\va_{i,t})-\vz_{i,t+1}\|_2\leq\epsilon_{\mathrm{dyn}},
\]
and that $g(\cdot,\va)$ is $L$-Lipschitz on all true and predicted states
being compared, with $L\geq0$. Then
\[
\max_i\|\widehat{\vz}_{i,T}-\vz_{i,T}\|_2
\leq\epsilon_{\mathrm{dyn}}\sum_{t=0}^{T-1}L^t.
\]
\end{lemma}

\begin{proof}
The assumptions give $e_{t+1}\leq Le_t+\epsilon_{\mathrm{dyn}}$ for
$e_t=\max_i\|\widehat{\vz}_{i,t}-\vz_{i,t}\|_2$, with $e_0=0$.
Iterating proves the claim for every $L\geq0$, including $L=1$.
\end{proof}

\subsection{Connection to a Task-Aligned Terminal Cost}
\label{app:task_cost}

\begin{corollary}[Task-cost regret]
\label{cor:task_cost}
Under Theorem~\ref{thm:planning_stability}, suppose
$J_{\mathrm{task}}:\sC\to\R$ satisfies
$|J_{\mathrm{task}}(i)-d_{\mP}(i,g)|\leq\epsilon_{\mathrm{anchor}}$
for every $i\in\sC$, where $\epsilon_{\mathrm{anchor}}\geq0$. Then
\[
J_{\mathrm{task}}(\widehat i)-\min_{i\in\sC}J_{\mathrm{task}}(i)
\leq2(\epsilon_{\mathrm{plan}}+\epsilon_{\mathrm{anchor}}).
\]
\end{corollary}

\begin{proof}
By anchor alignment and the planning regret bound,
\[
\begin{aligned}
J_{\mathrm{task}}(\widehat i)-\min_{i\in\sC}J_{\mathrm{task}}(i)
&\leq d_{\mP}(\widehat i,g)-\min_{i\in\sC}d_{\mP}(i,g)
     +2\epsilon_{\mathrm{anchor}}\\
&\leq2(\epsilon_{\mathrm{plan}}+\epsilon_{\mathrm{anchor}}).
\end{aligned}
\]
\end{proof}

\begin{figure}[h]
\centering
\includegraphics[width=0.53\textwidth]{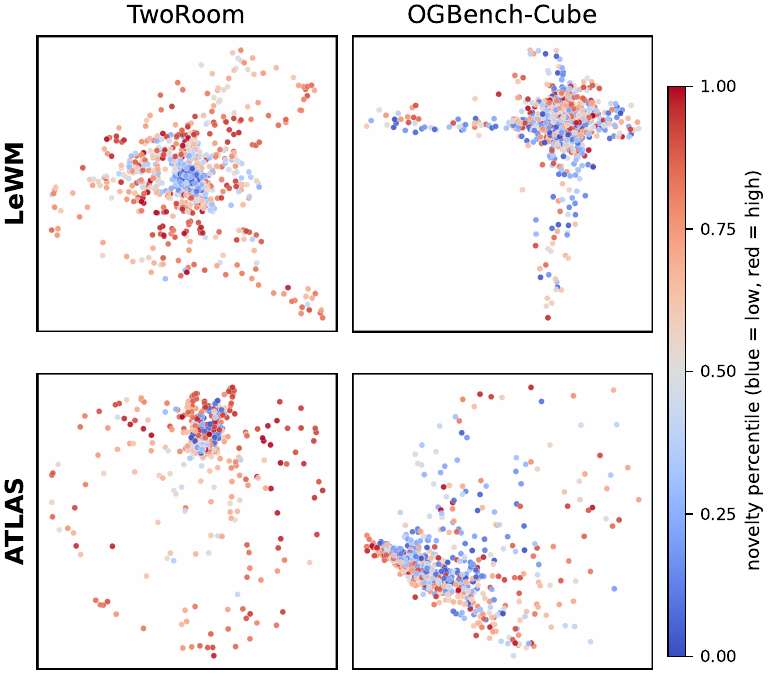}
\caption{\textbf{Planning-latent geometry on TwoRoom and OGBench-Cube.}
Two-dimensional PCA projections of $z$ for LeWM (top) and ATLAS (bottom),
companion to the PushT panels in \Cref{fig:ood-scatter}. Each point is a
state, colored by its model-independent true-state novelty percentile, from
lower novelty in blue to higher novelty in red. The separation between low-
and high-novelty states is less pronounced than on PushT, but the ATLAS
projections still place high-novelty states farther from the in-distribution
cluster than LeWM does.}
\label{fig:ood-scatter-appendix}
\end{figure}

\begin{figure*}[h]
\centering
\includegraphics[width=\textwidth]{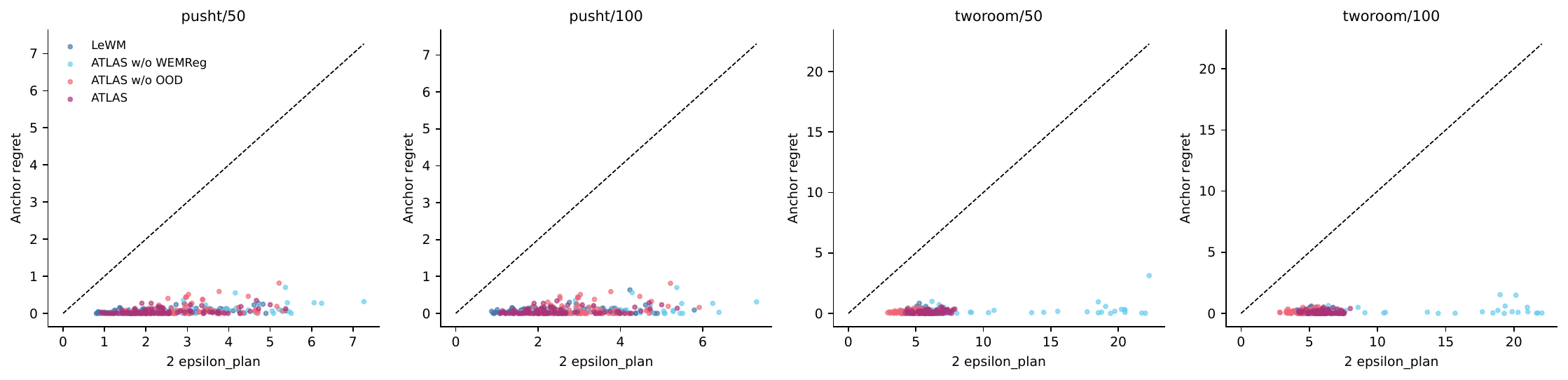}
\caption{\textbf{Bound versus regret.}
Each point is one arm-query evaluation; the dashed line is equality. The conservative bound and absent sufficient-margin events limit quantitative interpretation.}
\label{fig:regret}
\end{figure*}

\begin{figure*}[t]
\centering
\includegraphics[width=\textwidth]{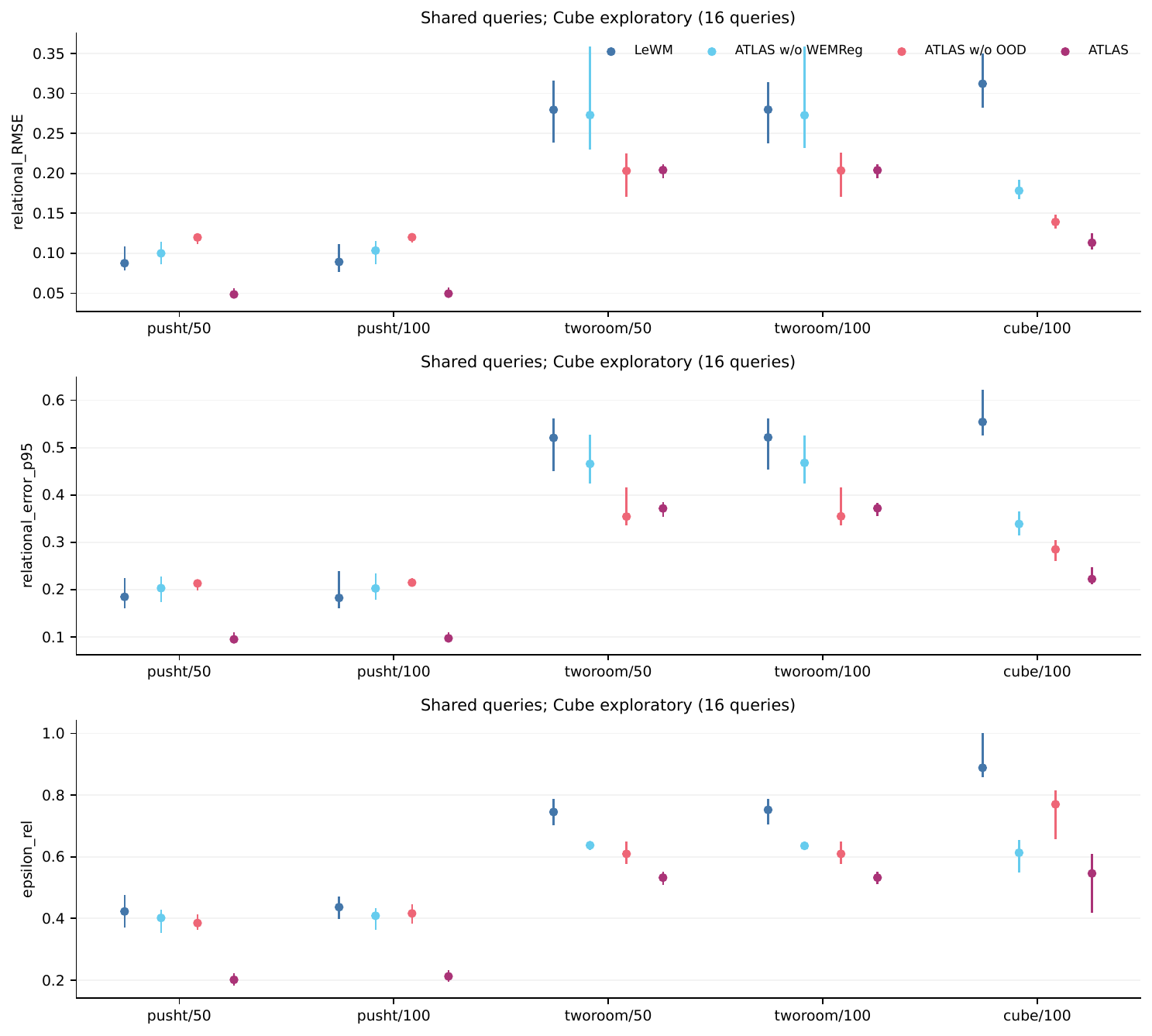}
\caption{\textbf{Relational metrics.}
Across-query medians and paired-bootstrap percentile intervals for within-query pairwise RMSE, 95th-percentile distortion, and maximum distortion. PushT/TwoRoom each use 100 queries, Cube 16 exploratory queries. All four arms use controlled checkpoints.}
\label{fig:reme}
\end{figure*}

\begin{figure*}[t]
\centering
\includegraphics[width=\textwidth]{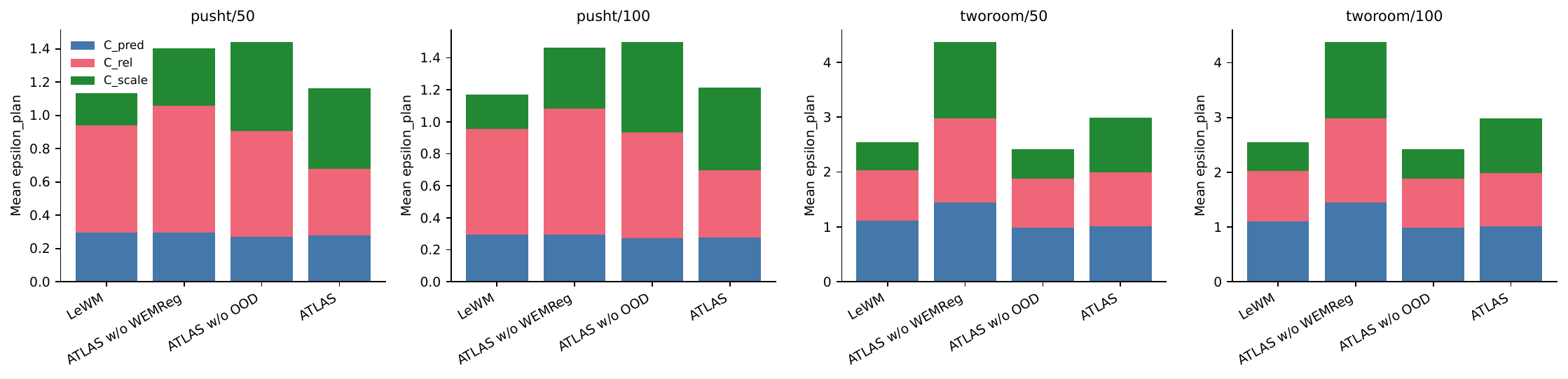}
\caption{\textbf{Theorem 3 decomposition.}
Bars show mean prediction, relational, and scale contributions to epsilon plan on the original queries. This is a structural decomposition, not a quantitative planning-success predictor.}
\label{fig:theo3}
\end{figure*}
% \section{Additional qualitative case studies}
% \label{app:cases}
% \Cref{fig:cases-appendix} shows the PushT and OGBench-Cube counterparts of the TwoRoom case
% study in \Cref{fig:cases}: one selected episode per task in which LeWM fails and ATLAS
% succeeds. As in the main text, these are qualitative success/failure pairs, not
% representative samples of all evaluation episodes.

\begin{figure}[h]
\centering
\includegraphics[width=\linewidth]{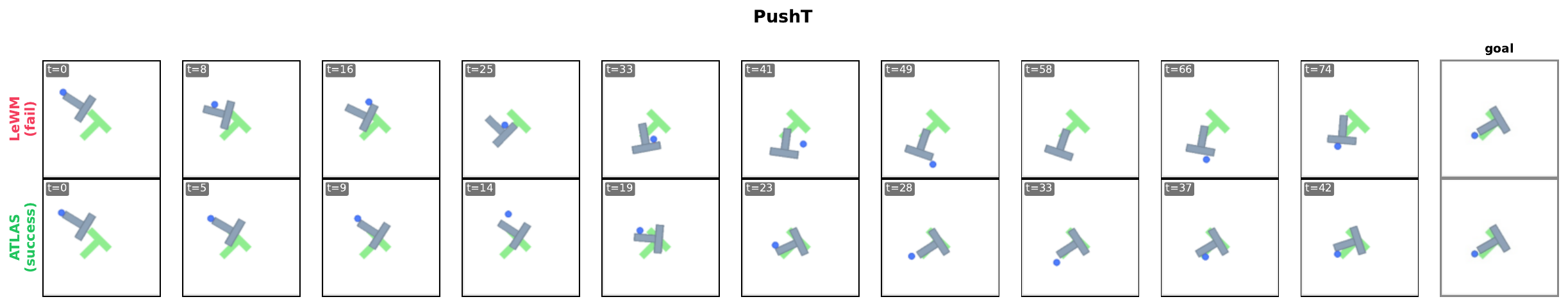}\\[4pt]
\includegraphics[width=\linewidth]{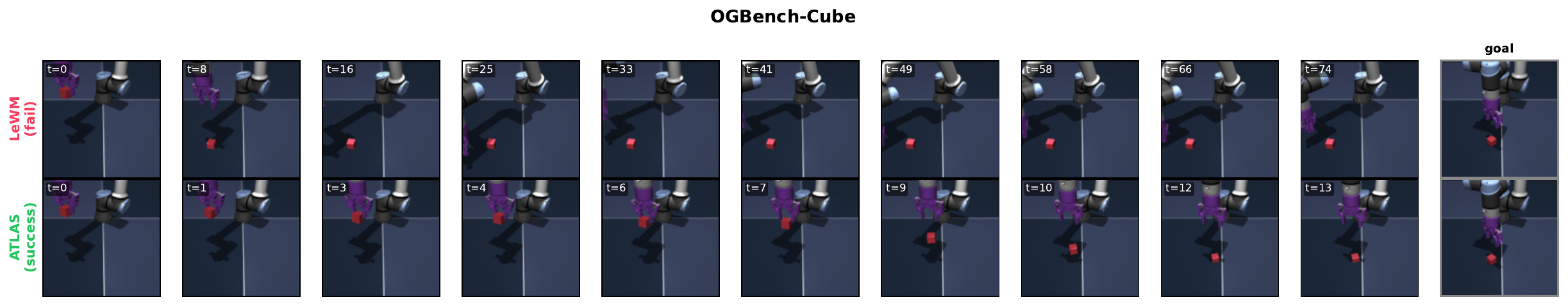}
\caption{\textbf{Selected closed-loop planning trajectories on PushT (top) and
OGBench-Cube (bottom).} One episode per task in which LeWM fails and ATLAS succeeds (LeWM
above ATLAS in each panel; frame indices indicate elapsed steps, goal in the rightmost
column). On PushT, ATLAS rotates and aligns the T-shaped block with the target pose; on
OGBench-Cube, ATLAS carries the cube to the target configuration.}
\label{fig:cases-appendix}
\end{figure}

\end{document}